%% file: main.tex
\documentclass[12pt]{article}

\usepackage{amsmath}
\usepackage{amssymb}
\usepackage{amsthm}
\usepackage{times}
\usepackage{graphicx}
\usepackage{color}
\usepackage{xcolor}
\usepackage{multirow}
\usepackage[authoryear]{natbib}
\usepackage{rotating}
\usepackage{bbm}
\usepackage{latexsym}

\usepackage{dsfont}
\usepackage{algorithm}
\usepackage{subcaption}
\usepackage{bm}
\usepackage{booktabs}
\usepackage{mathtools}
\usepackage{tabularx}
\usepackage{colortbl}
\usepackage[framemethod=TikZ]{mdframed}
\usepackage{doi}
\usepackage{hyperref}
\usepackage[nameinlink,capitalize]{cleveref}

\hypersetup{
    colorlinks=true,
    citecolor=blue,
    urlcolor=blue,
    linkcolor=blue,
}

\newcommand{\captionfonts}{\normalsize}

\makeatletter
\long\def\@makecaption#1#2{%
  \vskip\abovecaptionskip
  \sbox\@tempboxa{{\captionfonts #1: #2}}%
  \ifdim \wd\@tempboxa >\hsize
    {\captionfonts #1: #2\par}
  \else
    \hbox to\hsize{\hfil\box\@tempboxa\hfil}%
  \fi
  \vskip\belowcaptionskip}
\makeatother

\graphicspath{{figures/}{../figures/}}
\setkeys{Gin}{draft=false}

\newenvironment{keywords}
  {\vspace{1ex}\noindent{\bf Keywords:} }
  {\par\vspace{1ex}}

\newcommand{\acks}[1]{%
\subsection*{Acknowledgments}
#1
}

\newtheorem{theorem}{Theorem}
\newtheorem{lemma}[theorem]{Lemma}

\newtheorem{maintheorem}{Theorem}

\theoremstyle{definition}
\newtheorem{definition}[theorem]{Definition}

\theoremstyle{remark}

\crefname{assumption}{assumption}{assumptions}
\Crefname{assumption}{Assumption}{Assumptions}
\crefname{maintheorem}{theorem}{theorems}
\Crefname{maintheorem}{Theorem}{Theorems}

\newcommand{\floor}{\mathrm{floor}}
\newcommand{\inv}{\mathrm{inv}}

\newcommand{\mathmod}{\mathrm{mod}}

\mdfdefinestyle{thmbox}{
  backgroundcolor=gray!8,
  linecolor=gray!45,
  linewidth=0.6pt,
  roundcorner=4pt,
  innertopmargin=8pt,
  innerbottommargin=8pt,
  innerleftmargin=10pt,
  innerrightmargin=10pt,
}

\definecolor{phi11bg}{gray}{0.93}
\definecolor{phi12bg}{gray}{0.97}
\definecolor{phi13bg}{gray}{0.93}
\definecolor{highlight}{RGB}{220, 238, 255}

\definecolor{stage1blue}{HTML}{DAE8FC}
\definecolor{stage2green}{HTML}{D5E8D4}
\definecolor{stage3orange}{HTML}{FFE6CC}
\colorlet{phi1bg}{stage1blue}
\colorlet{phi2bg}{stage2green}
\colorlet{phi3bg}{stage3orange}

\definecolor{paramHL}{HTML}{FFD966}
\definecolor{widthHL}{HTML}{E6CCFF}

\newcommand{\hp}[1]{{\setlength{\fboxsep}{1.2pt}\colorbox{paramHL}{\ensuremath{#1}}}}

\def\m{\mathbf{m}}

\def\x{\mathbf{x}}

\newcommand{\balpha}{\bm{\alpha}}
\def\y{\mathbf{y}}

\def\c{\mathbf{c}}

\begin{document}

\thispagestyle{empty}
\markboth{}{Neural Computation manuscript}

\ \vspace{20mm}\\

{\LARGE On Explicit Super-Expressive Approximation for Neural Networks}

\ \\
{\bf \large Feng-Lei Fan$^{\displaystyle 1}$, 
Ze-Yu Li$^{\displaystyle 2}$, 
Chen-Yu Wang$^{\displaystyle 2}$, 
and Jian-Jun Wang$^{\displaystyle 3,*}$}\\

{$^{\displaystyle 1}$Department of Data Science, City University of Hong Kong, 83 Tat Chee Avenue, Kowloon Tong, Kowloon, Hong Kong.}\\
{\texttt{fenglfan@cityu.edu.hk}}\\

{$^{\displaystyle 2}$Department of Mathematics, Chinese University of Hong Kong, University Avenue, Shatin, N.T., Hong Kong.}\\
{\texttt{1155184076@link.cuhk.edu.hk}; \texttt{WangChenyuCUHK@outlook.com}}\\

{$^{\displaystyle 3}$School of Civil Engineering and Architecture, Guangxi Minzu University, Nanning 530006, China.}\\
{\texttt{20250017@gxmzu.edu.cn}}\\

\ \\[-2mm]
{$^{\displaystyle *}$Corresponding author.}



\input{chapters/00_Abstract}

\newpage
\setcounter{page}{1}

\input{chapters/01_Introduction}

\input{chapters/02_MainResults}

\input{chapters/03_Theorem1}

\input{chapters/04_Theorem2}

\input{chapters/07_Discussion}

\acks{This work is supported in part by the Lagrange Mathematics and Computing Research Center, Huawei Technologies France, and the Startup Fund of the City University of Hong Kong.}

\newpage

\vskip 0.2in

\bibliographystyle{apalike}
\bibliography{main}

\end{document}

%% file: chapters/00_Abstract.tex
\begin{abstract}
In this work, we investigate the fixed-architecture neural network approximation with explicit parameter bounds and elementary activations. While prior work demonstrated super-expressive approximation using fixed-size networks, they lack quantitative and non-asymptotic characterizations of parameter magnitude with respect to the approximation error. We resolve this issue by introducing the Chinese Remainder Theorem as a constructive encoding mechanism. For Lipschitz continuous functions on $[0,1]^D$, we construct a width-$\max\{D,4\}$, depth-$5$ network with explicit parameter-error trade-offs. For H\"older-smooth functions in $C^{r,\gamma}_A\left([0,1]^D\right)$, our fixed network of width $\max\{2D,\ D+5N+1\}$ and depth $r + 9$ achieves the parameter magnitude $\mathcal{P}$ bounded by $\log_2 \mathcal{P}=\mathcal{O}\bigl(\varepsilon^{-2D/(r+\gamma)}\log(1/\varepsilon)\bigr)$. This is the dual result compared to those in the parameter-bounded and architecture-unbounded paradigm.
\end{abstract}

\begin{keywords}
Neural network approximation, fixed-architecture neural networks, elementary super-expressive activation, H\"older-smooth functions
\end{keywords}

%% file: chapters/01_Introduction.tex
\section{Introduction}
\subsection{Background and Our Result}
Over the past few years, deep neural networks have achieved tremendous success in artificial intelligence \citep{babaiee2025master,Patrick2024Scaling} and data science \citep{Wu2025Data,Adaleta2024time}, motivating extensive research to establish the theoretical foundation of deep learning. In this context, a substantial body of research explores the approximation capacity
of deep neural networks to reveal the strong power of deep learning. These works investigate how well a network can express the target function via the composition of affine linear transforms and nonlinear activations. Typically, the approximation error is characterized in terms of the total number of parameters, the width and depth, or the number of neurons in a neural network. In particular, the approximation of several important basic function spaces has been well studied, \textit{e.g.}, continuous function space \citep{yarotsky2018Optimal,Shen2020deep}, $C^s$ function space \citep{yarotsky2020phase,Lu2021deep}, Sobolev space \citep{yarotsky2017error,hon2021simultaneous}, Besov space \citep{siegel2023optimal}, and Korobov space \citep{moise2022shallow,YANG2024near}. Although these studies have reached (nearly) optimal approximation rates, the size of the network typically grows exponentially as the error goes small.

One special direction of approximation results utilizes super-expressive activation functions, which allow to approximate a class of functions by a fixed-architecture network independent of the function and the error. That is, without enlarging the size of the network, one can achieve arbitrary accuracy by merely adjusting the weights, where multiple activations can be applied. The first fixed-architecture approximation result was given by \citep{maiorov1999lower}, which established the existence of super-expressive activations. The key idea is to construct a dense countable subset in $C([0,1])$ and then reduce the multivariate case to the univariate one via the Kolmogorov Superposition Theorem (KST) \citep{Kol1957}. Nevertheless, the super-expressive activation here is quite complex and has no closed-form formulation. Following the idea of KST, this computational intractability was solved by \citep{yarotsky2021elementary}.  In \citep{yarotsky2021elementary}, it was shown that any $C\left(([0,1]^d\right))$ function can be approximated by a fixed-size network with multiple elementary activation functions. In \citep{shen2022fixed}, the authors utilized the elementary universal activation function (EUAF) and constructed a fixed network to approximate any continuous functions. This work explicitly showed that an EUAF network requires only $\mathcal{O}(1)$ depth and $\mathcal{O}(d^2)$ width to achieve super-expressiveness, using irrational winding in Diophantine approximations and KST. Furthermore, \cite{wang2025peuaf} proposed a parametric EUAF (PEUAF) based fixed-architecture network with a similar super-expressive result. They also showed that such super-expressive networks are trainable in practice.

However, achieving super-expressiveness does come with a cost. Once the width and depth are fixed, the approximation burden is transferred to the magnitude of the parameters. Hence, establishing a quantitative and non-asymptotic parameter bound \textit{w.r.t.} the approximation error is crucial in fixed-architecture approximation.  \cite{beknazaryan2022neural} gives an explicit super-expressive approximation of H\"older continuous functions by a refined Kronecker's Theorem. Although the network size and parameter magnitude are explicit, the activation is recalibrated as the error tolerance changes. That is, \cite{beknazaryan2022neural} trades the growth of architecture for growth of activation complexity. To the best of our knowledge, existing work on fixed-architecture network approximation lacks quantitative and non-asymptotic bounds on parameter magnitudes \textit{w.r.t.} approximation error, and those that do provide relevant results rely either on error-dependent activation functions or on prior encoding of the target function in the initial input. The main reason is that they utilized the Diophantine approximation or the KST, which contains unknown prefactors hard to be explicitly formulated. 

In this work, we propose a quantitative and non-asymptotic fixed-architecture approximation of Lipschitz continuous and H\"older-smooth functions with the parameter magnitude explicitly characterized by the approximation accuracy. Our central mechanism is the introduction of the Chinese Remainder Theorem (CRT). Given a partition of the input domain indexed by coprime integers and the quantization of the function value indexed by integers, we leverage CRT to explicitly associate each cell to a quantized value. Then a fixed-architecture network with multiple elementary super-expressive activation functions to reconstruct the target.
Our main results are as follows:

{\fontsize{9}{10}\selectfont $\bullet$} 
\textbf{Approximation of Lipschitz functions.} For any Lipschitz function in $[0,1]^D$, we construct a network of width $\max\{D,4\}$ and depth $5$ to not only prove the existence of such a fixed-size network but also provide explicit parameter magnitudes $\mathcal{P}$ satisfying $\log_2\mathcal{P}\le O\bigl(\varepsilon^{-2D}\log(1/\varepsilon)\bigr)$ that relate the target accuracy $\varepsilon$ to the parameter magnitude of the network. To the best of our knowledge, this is the first explicit relation made in the area of super-expressive approximation.

{\fontsize{9}{10}\selectfont $\bullet$} 
\textbf{Approximation of H\"older-smooth functions.} For any H\"older-smooth function $f \in C^{r,\gamma}_A([0,1]^D)$, we construct a rational gridwise polynomial surrogate via a local Taylor expansion. This approximation is achieved using a fixed-architecture network with width $\max\{2D,\ D+5N+1\}$ and depth $r + 9$. Crucially, higher target smoothness mitigates the growth of the required parameter magnitude, which we bound by $\log_2\mathcal{P}\le O\bigl(\varepsilon^{-2D/(r+\gamma)}\log(1/\varepsilon)\bigr)$. This demonstrates that the smoothness of the target function improves the scaling law, reducing the complexity exponent from $2D$ to $2D/(r+\gamma)$.

Traditional approximation results provide the change in architecture with respect to the error while the parameter magnitude is fixed. Our result can be regarded as dual to the traditional paradigm, which is instrumental in completing the picture of super-expressive approximation.


\subsection{Related Work}

\textbf{Classical neural network approximation}.
The approximation capacity of neural networks has been a central theme in deep learning theory and has seen a flourishing body of work in recent years. Early works in universal approximation \citep{cybenko1989approximation,hornik1989universal,hornik1991approximation,leshno1993nonpolynomial,barron1993universal} showed that sufficiently wide or deep fully-connected neural networks can approximate any continuous function on compact sets without characterizing the approximation rate. The universal approximation theorem of other networks is also widely studied. For example, \cite{zhou2020cnn, zhou2020downsampling, yu2023deep,li2025fullyconv} demonstrated that deep convolutional networks can  approximate arbitrary continuous functions on compact sets with sufficient depth, showing that locality and weight sharing preserve universality. In addition, \citep{yun2019transformers} proved that transformers universally approximate continuous permutation-equivariant sequence-to-sequence maps, and \cite{lin2018Resnet} showed that the ReLU activated ResNet with one neuron in each layer is a universal approximator. Besides, the universal approximation of ResNet can also be established through the dynamic system \citep{Li2023Deep,Li2019stochastic}.

While the classical universal approximation theorems mainly focus on the existence of neural network approximants, a more quantitative line of research has subsequently emerged, aiming to characterize the approximation error in terms of network width, depth, parameter count, and the number of neurons. The error estimate using the total number of parameters in ReLU networks has been well studied and has reached (nearly) optimal approximation rate for several basic function spaces, \textit{e.g.}, continuous functions \citep{yarotsky2018Optimal}, smooth functions \citep{yarotsky2020phase,Petersen2018Optimal}, and Sobolev space \citep{hon2021simultaneous,yarotsky2017error}.  Errors in terms of width and depth further motivated the works on the non-asymptotic approximation, \textit{e.g.,} continuous functions\citep{Shen2020deep,Shen2022optimal}, $C^s$ functions \citep{Lu2021deep}, and functions in Korobov space \citep{moise2022shallow, YANG2024near}, which not only achieve (nearly) optimal approximation error, but also are explicit without unknown prefactors. Although the errors are (nearly) optimal in classical network approximation, the size of the networks increases exponentially as the error goes small. 

\vspace{+5mm}

\noindent\textbf{Super-expressive activation and fixed-architecture approximation}. Designing super-expressive activations is one of the directions to improve the approximation rate, where a class of target functions can be approximated by a neural network using a finite family of activation functions and fixed architecture only depending on the input dimension. The initial idea of super-expressive approximation comes from the KST \citep{Kol1957}, which shows that any continuous function $f\in C\left([0,1]^d\right)$ can be represented as $f(\mathbf{x})=\sum_{i=0}^{2d}g_i \left(\sum_{j=1}^dh_{i,j}(x_j)\right)$. Note that the KST can be viewed as the network composition with activations depending on the target function. See \citep{igelnik2003kolmogorov,kuurkova1991kolmogorov,kuurkova1992kolmogorov,maiorov1999lower} for further studies on the connection of KST and neural networks. Especially, in \citep{maiorov1999lower}, the authors used the KST to design a special activation and construct a fixed-size network with $\mathcal{O}(d)$ neurons, which can approximate any continuous function on $[-1,1]^d$ with an arbitrary error. Since this activation is complicated and has no closed form, this network cannot be applied in practice. Later, several works \citep{guliyev2016single,guliyev2018singlefixed,ismailov2014approximation,guliyev2018fixedweights} further studied similar activations. For a long time, the computational intractability in super-expressive approximation remains a hard issue. Another line of research explored weaker forms of super-expressiveness with elementary activations, where the network size has to grow for higher accuracy, but much slower than the power laws in the classical neural network approximation theory. For example, \cite{yarotsky2020phase} showed that a network with ReLU and $\operatorname{sin}$ activations can approximate Lipschitz functions with exponential error rate \textit{w.r.t.} the number of parameters. Besides, \cite{shen2021threehidden} also demonstrated that a three-layer network with $\lfloor\cdot\rfloor$, $2^x$, and $\mathbf{1}_{x\geq0}$ as activations
can approximate Lipschitz functions with a similar error.

The computational intractability of super-expressive approximation was solved in \citep{yarotsky2021elementary}, where a fixed-size network was constructed with multiple elementary activation functions, \textit{e.g.}, $\operatorname{sin}$ \& $ \operatorname{arcsin}$, $\lfloor\cdot\rfloor$ \& a non-polynomial analytic function, or a specific $C^1$ piecewise elementary function, to approximate any $C\left([0,1]^d\right)$ functions within arbitrarily small error. It was also shown in \cite{yarotsky2021elementary} that most of the practically used activations are not super-expressive. However, this result does not give explicit characterization of the size dependence on $d$. To make the dependence clear, \cite{shen2022fixed} showed that a EUAF network of width $36d(2d+1)$ and
depth $11$ can approximate any continuous function on a $d$-dimensional closed cube. Later, \citep{wang2025peuaf} showed that a fixed-architecture network with with parametric EUAF is super-expressive are trainable in practice. To further give the parameter magnitude in terms of approximation accuracy, \cite{beknazaryan2022neural} utilized  piecewise exponential and $\lfloor\cdot\rfloor$ activations to give a super-expressive approximation of H\"older continuous functions with explicit integer weights, where the activation is not fixed in the usual sense and varies with the target accuracy. In addition, \cite{bournez2025universal} constructed a fixed ReLU network to approximate arbitrary continuous functions, but it requires the encoding of the target function to be given in the initial input.

%% file: chapters/02_MainResults.tex
\section{Preliminaries}\label{sec:notation}

We collect notations and definitions used throughout the paper.  We first introduce the basic function spaces, including Lipschitz continuous functions, H\"older-smooth functions, and gridwise polynomials.

\begin{definition}[Lipschitz Continuous Function]
\label{def:lipschitz}
Let $A>0$. Define
\begin{equation}
\operatorname{Lip}_A(\Omega)
:=\bigl\{f:\Omega\to\mathbb R:\ |f(\x)-f(\y)|\le A\|\x-\y\|_\infty\bigr\}.
\end{equation}
We say $f\in\operatorname{Lip}_A(\Omega)$ is Lipschitz continuous on $\Omega$.
\end{definition}

\begin{definition}[H\"older-Smooth Function]
\label{def:holder}
Let $A>0$, $r\in\mathbb N$, $\gamma\in(0,1]$. Define
\begin{equation}
C^{r,\gamma}_A(\Omega)
:=\biggl\{f:\Omega\to\mathbb R:\ 
\sup_{\substack{g\in\partial^{r}f\\ \x\ne\y}}
\frac{\bigl|g(\x)-g(\y)\bigr|}{\|\x-\y\|_\infty^{\gamma}}\le A\biggr\},
\end{equation}
where $f$ has continuous partial derivatives up to order $r$ and $\partial^{r}f$ is the set of up to order-$r$ partial derivatives. We say $f\in C^{r,\gamma}_A(\Omega)$ is H\"older-smooth on $\Omega$.
In particular $C^{0,1}_A(\Omega)=\operatorname{Lip}_A(\Omega)$.
\end{definition}

\begin{definition}[Gridwise Polynomials with Rational Coefficients]
\label{def:rgpp}
Let $\mathrm{Poly}_{D,r}(\mathbb{Q})$ denote the  polynomials at degree $r$ with rational coefficients:
\begin{equation}
\mathrm{Poly}_{D,r}(\mathbb{Q}) := \bigg\{ \sum_{\substack{ \\ |\balpha|\le r}} c_{\balpha}\x^{\balpha} \;\bigg|\; c_{\balpha} \in \mathbb{Q} \bigg\},
\end{equation}
where $\balpha := (\alpha_1,\alpha_2,...,\alpha_D) \in\mathbb N_0^D$ and $\x^{\balpha} := \prod_{d=1}^D x_d^{\alpha_d}$. 

Consider the uniform partition of the domain $\Omega = [0,1]^D$ into $M^D$ grids  indexed by $\mathbf{m} \in \{0,\dots,M-1\}^D$:
\begin{equation}
\Omega_{\mathbf m} := \prod_{d=1}^{D} \left[ \frac{m_d}{M}, \frac{m_d+1}{M} \right),
\end{equation}
with the convention that the right endpoint is included when $m_d = M-1$. A function $f: \Omega \to \mathbb{R}$ is called a gridwise polynomial with rational coefficients if, for every grid $\Omega_{\mathbf m}$, there exists a polynomial $P \in \mathrm{Poly}_{D,r}(\mathbb{Q})$ such that $f(\x) = P(M\x - \mathbf{m})$ for all $\x \in \Omega_{\mathbf m}$.
\end{definition}

Gridwise polynomials form an important class of functions. They are capable of capturing local variations, boundary layers, spikes, and piecewise regular structures, while avoiding the rigidity of a single global polynomial. Moreover, they have a natural connection to scientific computing: splines and finite element functions are precisely built from local polynomial pieces over grids. In fact, gridwise polynomials are the Cartesian-grid analogue of the local polynomial spaces used in finite element discretizations \citep{ciarlet2002finite}, which are fundamental in modeling a wide range of physical phenomena, including elasticity, heat transfer, and fluid flow. On structured Cartesian or hexahedral meshes, this is exactly the type of polynomial structure captured by the gridwise class \citep{reddy2010finite}.

We introduce the bit length to measure the bits required to encode a rational number, which follows the standard convention in algorithmic information theory and computational complexity. 
\begin{definition}[Bit Length of a Rational Number \citep{schrijver1986theory}]
\label{def:rational-bit-length}
For $q\in\mathbb Q$, write $q$ in reduced form:
\begin{equation}
\label{eq:rational-st}
q=\frac{s(q)}{g(q)},
\end{equation}
where $s(q)\in\mathbb Z$, $g(q)\in\mathbb N_+$, and
$\gcd(|s(q)|,g(q))=1$.  For $q=0$, we use the convention
$s(q)=0$ and $g(q)=1$.  Define the rational bit length by
\begin{equation}
\label{eq:rational-bit-length}
\mathtt{bits}(q)
:=
\bigl\lceil \log_2(|s(q)|+2)\bigr\rceil
+
\bigl\lceil \log_2(g(q)+1)\bigr\rceil .
\end{equation}
\end{definition}

We then introduce the activation functions used in this work, which are all simple and have been widely used in literature. See Figure~\ref{fig:activation-primitives} for visualization.

\begin{definition}[Floor Activation \citep{shen2021floorrelu}]
\label{def:fractional-part-activation}
The floor activation is given by
\begin{equation}
    \rho_{\floor}:\mathbb R\to\mathbb Z,
    \qquad
    \rho_{\floor}(t):=\lfloor t\rfloor ,
\end{equation}
where $\lfloor t\rfloor$ denotes the largest integer not exceeding $t$.
Equivalently, $\rho_{\floor}(t)=k$ whenever $t\in[k,k+1)$ for some
$k\in\mathbb Z$. 
We also write $\rho_{\mathrm{frac}}(t)=t-\lfloor t\rfloor$ as a derived notation
\end{definition}

\begin{definition}[RePU Activation, \citep{shen2023repu}]
\label{def:repu-activation}
For an integer $s\in\mathbb N_+$, the \emph{rectified power unit}
$\mathrm{RePU}_s$ is defined as
\begin{equation}
\rho_s(t)=
\begin{cases}
t^s, & t\ge 0,\\
0, & t<0.
\end{cases}
\end{equation}
In particular, when $s=1$, $\rho_s$ is the usual ReLU activation; when $s=2$, it is often called the ReQU activation. We can use RePU activations as exact algebraic primitives for multiplication.  More precisely, we define 
\begin{equation}
    \label{eq:multiple}
\mathtt{Mult}(u,v)
:=
\big(\rho_2(u+v)+\rho_2(-u-v)
-\rho_2(u-v)-\rho_2(v-u)\big)/
4.
\end{equation}
\end{definition}

\begin{definition}[Reciprocal Activation  \citep{boulle2020rational}]
\label{def:safe-reciprocal-activation}
Define $\rho_{\rm \inv}:\mathbb R\to\mathbb R$ by
\begin{equation}
\rho_{\rm \inv}(t)
=
\begin{cases}
1/t, & t\ge 1,\\
2-t, & t<1,
\end{cases}
\end{equation}
which, together with \(\rho_{\mathrm{floor}}\),  enables the exact modular operation
$$
M \bmod p = M - p \rho_{\mathrm{floor}}\bigl(M \rho_{\mathrm{inv}}(p)\bigr),
$$
for every $M\in\mathbb{Z}$ and $p\in \mathbb{N}^+$.

\end{definition}

\begin{figure}[htbp]
    \centering
    \includegraphics[width=0.7\linewidth]{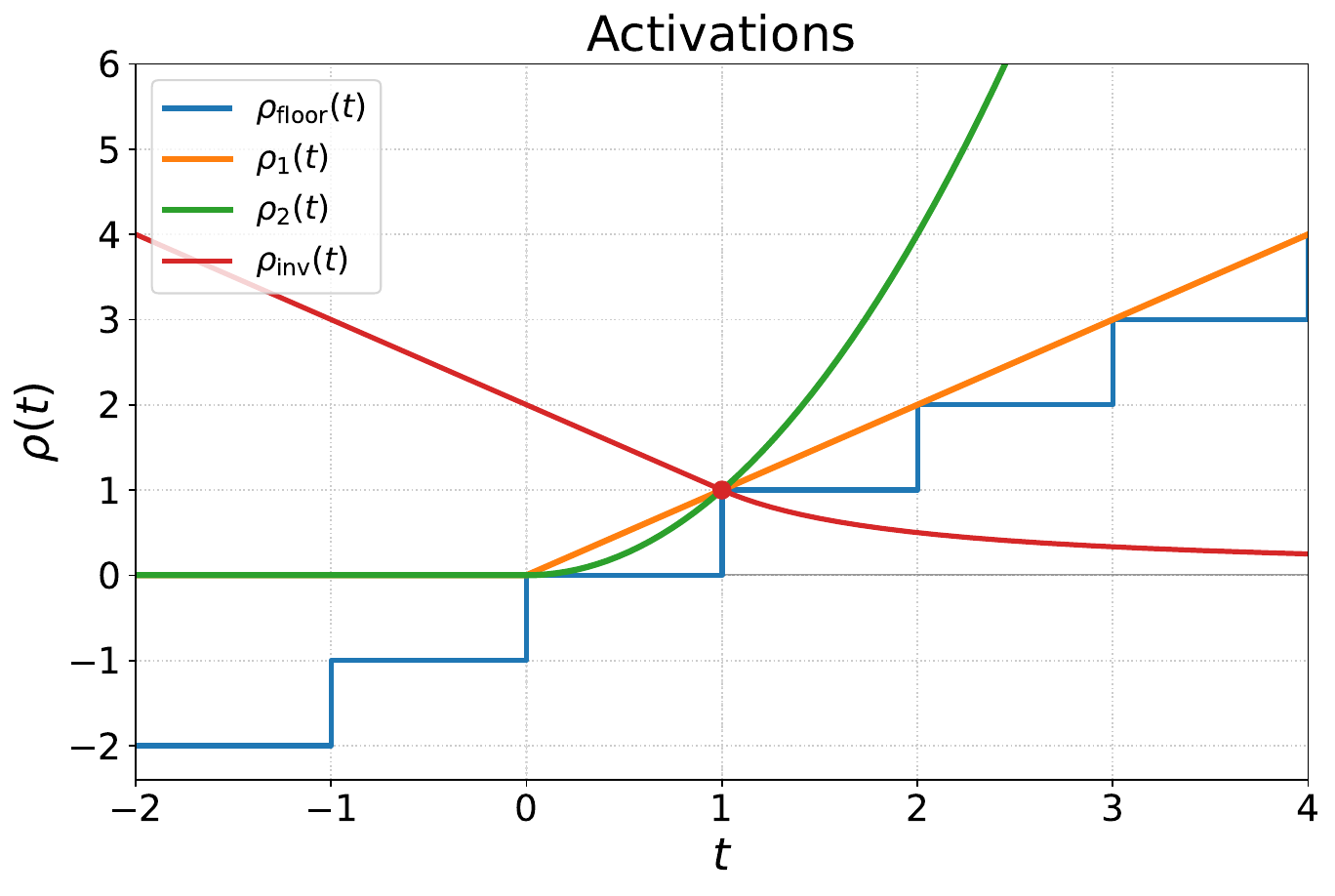}
        \vspace{-0.5cm}
    \caption{
    Visualization of  the floor  activation $\rho_{\floor}(t)$, the RePU activation $\rho_1(t),\rho_2(t)$, and the reciprocal activation $\rho_{\inv}$.
    }
    \label{fig:activation-primitives}
\end{figure}


%% file: chapters/03_Theorem1.tex
\section{Approximation of Lipschitz Functions}
\label{sec:section3}

In this section, we propose a fixed-architecture construction for approximating Lipschitz continuous functions with explicit parameter growth. This bridges a critical gap in super-expressive approximation: previous works either established mere existence without quantifying parameter scales \citep{yarotsky2021elementary, shen2022fixed}. Unlike their work that used irrational winding and KST, we apply CRT to give an explicit construction. For the first time, our theorem provides an explicit bound on the parameter magnitude in terms of the approximation error.

\begin{maintheorem}[Lipschitz Continuous Function]
\label{thm:main-theorem1}
Let $\Omega=[0,1]^D$ and  $f\in \operatorname{Lip}_A(\Omega)$. Then for arbitrary $\varepsilon>0$, there exists a network $\Phi$ of width $\max\{D, 4\}$ and depth $5$ with $\rho_{\floor}, \rho_1,\rho_2$, and $\rho_{\inv}$ as activations, such that
\begin{equation}
\label{eq:main-continuous-integral-error}
\bigl|\Phi(\x)-f(\x)\bigr|
\
\le
\varepsilon.
\end{equation}
Moreover,  the parameter magnitude satisfies
\begin{equation}
\label{eq:main-continuous-parameter-bound}
\mathcal P(\Phi)\;\le\;\prod_{i=1}^{M_\varepsilon^{D}}\Bigl(1+i\,(J_\varepsilon+1)\,(M_\varepsilon^{D})!\Bigr),
\end{equation}
where $B =  \|f\|_{L^\infty(\Omega)}$,
$ M_\varepsilon:= \left\lceil 2A/\varepsilon \right\rceil$, and
$ J_\varepsilon:=\left\lceil 4B/\varepsilon\right\rceil
$.
\end{maintheorem}

\subsection{Idea of the Construction in Theorem \ref{thm:main-theorem1}}
\label{subsec:proof-maintheorem1}

Here, we show the idea of proving Theorem \ref{thm:main-theorem1} step by step. The construction follows a fixed-architecture finite-fitting strategy, which has been proposed in \citep{shen2021threehidden} for a weaker form of super-expressiveness. Our two novel main ideas are:
1) First, a continuous function is simplified to a finite number of representative values through quantization. 
2) Then, unique pairwise coprime integers label the domain grids, enabling CRT to bridge the domain and function value. We take the following four steps:

\emph{1. Divide and reduce the target to finitely many codes.}
Choose $M\in\mathbb N_+$, we use $\mathbf m=(m_1,\ldots,m_D)\in \{0,\ldots,M-1\}^D$ as the coordinate of a grid.  For each $\mathbf m$, define
\begin{equation}
Q_{\mathbf m}
:=
\prod_{d=1}^D
\left[
\frac{m_d}{M},
\frac{m_d +1}{M}
\right),
\end{equation}
with the convention that an interval whose right endpoint is $1$ is closed,
and otherwise is half-open.  Then
$\bigcup_{\mathbf m\in\{0,\ldots,M-1\}^D}
Q_{\mathbf m}$ forms a partition of $\Omega=[0,1]^d$.  Let
$
\x_{\mathbf m}
:=
\m/M
$
be the corner of $Q_{\mathbf m}$.

Now, we quantize the function value of $f$. Recall that $\Vert f \Vert_\infty = B$. Then we choose $J\in\mathbb N_+$ and set $\delta:=2B/J$. Thus, we have $J$ intervals $[j\cdot \delta-B, (j+1)\cdot \delta-B],  j=0, 1, \cdots, J-1$. Due to the Lipschitz continuity of $f$, for any $J$, as long as $M$ is sufficient large, there always exists
\begin{equation}
-B+j\delta
\le
f(\x_{\m})
\leq
-B+(j+1)\delta,
\end{equation}
for some $j$. Thus, we establish the one-to-one mapping between the grid $\m$ and the $j$-th quantization level. For convenience, we refer to it as $j_\m$. Our approximation is to use a piecewise constant function $\bar{f}(\x)$ over $Q_\m$ to approximate the target function $f$. In each grid $Q_\m$, the constant is the bottom value of the quantization interval. Mathematically, we have
\begin{equation}
\bar{f}(\x)
:=
-B+ j_{\mathbf m} \delta
\end{equation} for $\forall \x \in Q_\m$. 
Then, for every $x\in Q_{\mathbf m}$,
\begin{equation}
\label{eq:main1-proof-idea-local-error}
|\bar{f}(\x)-f(\x)|
\le
|\bar{f}(\x)-f(\x_{\m})|
+
|f(\x_{\m})-f(\x)|
\le
\delta+ A/M.
\end{equation}
See Figure~\ref{fig:CodeAssignment} for our partition of the input domain and the quantization for the function value of $f$.


\begin{figure}
    \centering
    \includegraphics[width=0.88\linewidth]{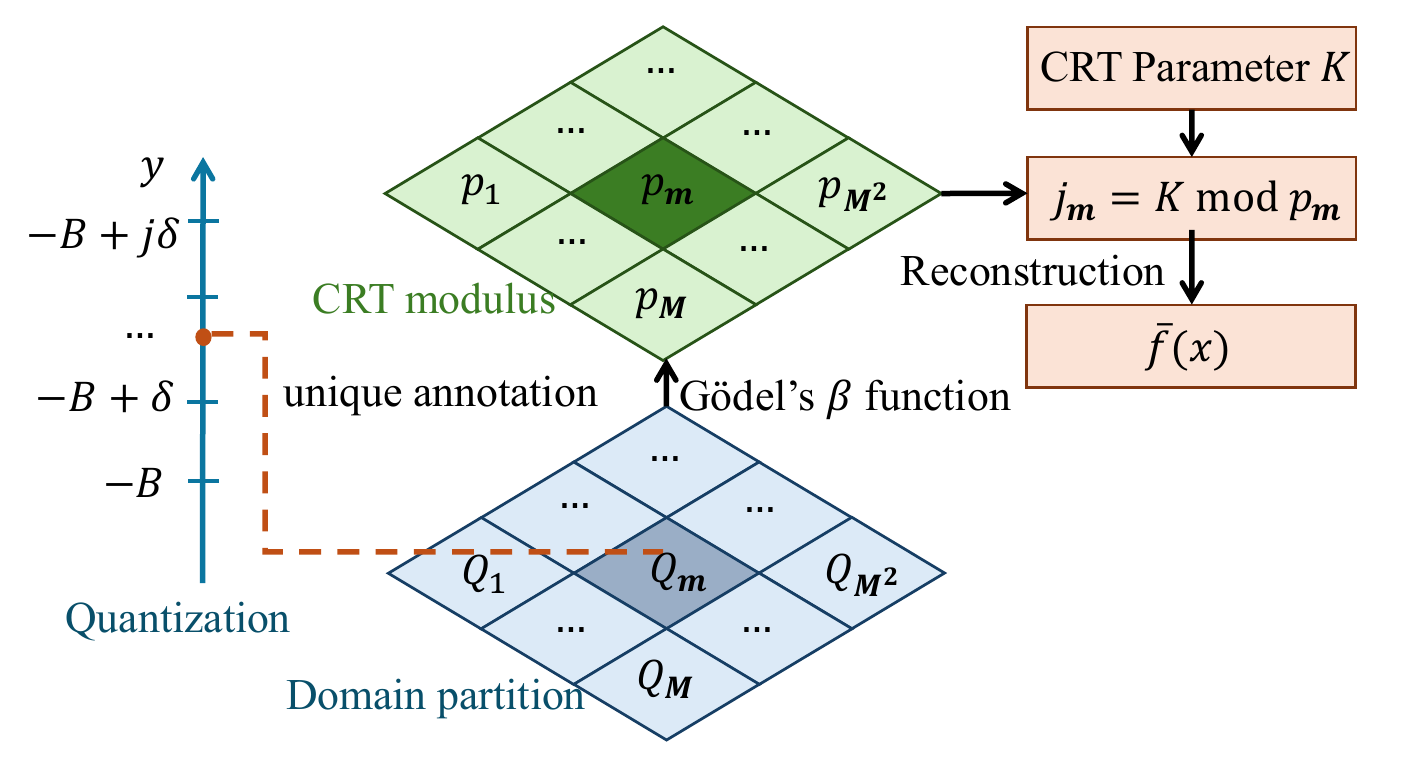}
    \caption{Illustration of Theorem \ref{thm:main-theorem1} pipeline.}
    \label{fig:CodeAssignment}
\end{figure}

\emph{2. Extract the active grid index off the grid boundary.}
For $\x\in\Omega$, each coordinate $x_d$ lies strictly inside one of the intervals $[m_d/M,(m_d+1)/M]$. We derive the index of the grid where $\x$ belongs by the following equation:
\begin{equation}
    \mathbf{m}=[\rho_{\floor}(M x_1), \rho_{\floor}(M x_2), \cdots, \rho_{\floor}(M x_D)] \in \{0, 1, \cdots, M-1\}^D.
\end{equation}
Then, we generate a unique annotation for each grid based on its associated index:
\begin{equation}
\label{eqn:key}
\Phi_1(\x)
:=
1+\sum_{d=1}^D M^{d-1}\rho_{\floor}(M x_d),
\end{equation}
which is exactly the base-$M$ encoding of the active grid coordinate. The first
sub-network realizes $\Phi_1$.
In the activation model containing $\rho_{\floor}$, this sub-network is implemented by applying $D$ floor units to $M x_1,\ldots,M x_D$, followed by one affine combination.

\emph{3. Construct a group of modulus by CRT.}
Now, we assign a modulo for each grid based on its unique annotation. Set
\begin{equation}
c:=(J+1)M^{D}!,
\qquad
p(\x):=1+\Phi_1(\x)c,
\end{equation}
then $p(\x)>J$, and the integers $p(\x)$ and $p(\x')$ are pairwise coprime if $\x$ and $\x'$ are from two different grids. The coprimality of $p(\x)$ is formally confirmed by Lemma~\ref{lem:godel-crt-moduli-coprime}. Then by CRT, for $\x$ located at $Q_\m$, there exists an integer $K$ such that
\begin{equation}
K\equiv j_{\m} \pmod{p(\x)},
\end{equation}
Since $0\le j_\m<J<p(\x)$, the code $j_{\m}$ is exactly the standard residue of
$K \mod p(\x)$. Thus
\begin{equation}
\label{eq:crt-readout-residue}
j_{\m}
=
K\bmod p(\x)
=
K
-
p(\x)
\left\lfloor \frac{K}{p(\x)}\right\rfloor.
\end{equation}

The quotient term in \eqref{eq:crt-readout-residue} is realized by
\begin{equation}
q(\x)
:=
\rho_{\floor}
\!\left(
K\rho_{\inv}(p(\x))
\right).
\end{equation}
Since $K$ is a fixed target-dependent parameter, the factor
$K\rho_{\inv}(p(\x))$ is obtained by affine scaling after the inverse
activation.  Hence
$
q(\x)
=
\left\lfloor \frac{K}{p(\x)}\right\rfloor .
$

It remains to realize the product $p(\x)q(\x)$.  This product can not be produced by an affine layer because both factors are hidden-channel quantities. 
With the ReQU multiplication module $\mathtt{Mult}$,  the CRT module is
\begin{equation}
\label{eq:CRT-module}
\Phi_2(\x)
:=
K
-
\mathtt{Mult}(p(\x),q(\x)).
\end{equation}

\emph{4. Reconstruct the value and control the error.}
Finally, we use a network to represent the following function
\begin{equation}
\Phi_3(j_{\m})
:=
-B+ j_{\m}\delta ,
\end{equation}
which reconstructs the function value from the quantization level. Then 
\begin{equation}
\bar{f}(\x)
=
\Phi_3\circ\Phi_2\circ\Phi_1,
\end{equation}
concludes our construction.
See Figure~\ref{fig:CodeAssignment} for an illustration of the target network.

Based on the aforementioned construction, we can give the width and depth of the network and estimate the parameter bounds based on the error.

\subsection{Proof of Theorem \ref{thm:main-theorem1}}
\label{Sec:Beta-functionTheorem1}

In our proof idea, the moduli \(p(\x)\) must satisfy two requirements such that CRT can be applied to have a fixed-architecture network to realize an arbitrarily small error. First, each modulus \(p(\x)\) must be greater than the number of intervals \(J\). Second, \(p(\x)\) and \(p(\x')\) must be pairwise coprime for $\x$ and $\x'$ are from different grids. G\"odel's \(\beta\)-function provides a classical method for forming coprime numbers \citep{godel1931,smith2013godel}.

\begin{lemma}[Pairwise coprimality of \(p_i\)]
\label{lem:godel-crt-moduli-coprime}
Given \(J,M,D\in\mathbb N_+\), and set
\(
c:=(J+1)(M^D)! .
\)
Define
\begin{equation}
p_i:=1+ic,
\qquad i=1,\ldots,M^D .
\label{beta}
\end{equation}
Then \(p_1,\ldots,p_{M^D}\) are pairwise coprime.  Moreover,
\(p_i>J\) for every \(i=1,\ldots,M^D\).
\end{lemma}

\begin{proof}
i) First, since \(c=(J+1)(M^D)!\ge J+1\), we have
\(
p_i=1+ic>J,
\)
for $i=1,\ldots,M^D$. ii) Second, fix \(1\le a<b\le M^D\), and let \(g\in\mathbb N_+\) be a common divisor of \(p_a\) and \(p_b\), which means
\(
g\mid p_a\) and 
\(
g\mid p_b .
\)
Then \(g\) also divides their difference:
\(
g\mid (p_b-p_a) .
\)
Since
\(
p_b-p_a
=
(1+bc)-(1+ac)
=
(b-a)c,
\)
we get
\(
g\mid (b-a)c.
\)
Since \(p_a=1+ac\), any common divisor of \(p_a\) and \(c\) must divide
\(
p_a-ac=1.
\)
Hence \(\gcd(p_a,c)=1\).  
Because \(g\mid p_a\), it follows that
\(
\gcd(g,c)=1.
\)
Together with \(g\mid (b-a)c\), this implies
\(
g\mid (b-a).
\)
Since \(1\le b-a\le M^D-1\), the integer \(b-a\) divides \((M^D)!\). So we have
\(
(b-a)\mid c.
\)
Thus \(g\mid c\).  Together with \(g\mid p_a\), this gives
\(
g\mid (p_a-ac).
\)
Therefore \(g=1\).  Hence \(\gcd(p_a,p_b)=1\).  Since \(a<b\) is arbitrary, \(p_1,\ldots,p_{M^D}\) are pairwise coprime.
\end{proof}

\begin{lemma}[Chinese Remainder Theorem (CRT)]
\label{prop:crt}
Let \(p_1,\dots,p_N\) be pairwise coprime integers, and let
\(b_1,\dots,b_N\in\mathbb Z\).  Then the system of congruences
\begin{equation}
K\equiv b_j ~~\mathmod ~{p_j},
\qquad j=1,\dots,N,
\end{equation}
has a solution \(K\in\mathbb Z\).  Moreover, the solution is unique modulo \(
\prod_{j=1}^{N} p_j.
\)
Equivalently, there exists a unique residue class \(
[K]\in \mathbb Z/P\mathbb Z
\)
satisfying all the congruences, and one may choose its representative so that \(
0\le K<\prod_{j=1}^{N} p_j.
\)
\label{CRT}
\end{lemma}


\begin{proof}[Theorem \ref{thm:main-theorem1}]
Combining Lemma \ref{lem:godel-crt-moduli-coprime} and Proposition \ref{CRT}, it can be seen that Eq. \eqref{eq:crt-readout-residue} holds true.
The moduli \(p_1,\ldots,p_{M^D}\) are pairwise coprime.  Hence, there exists an integer \(K\) such that
\(
K\equiv j_\m \pmod{p_\m}.
\)
Since \(p_\m>J\) and \(j_\m\in\{0,\ldots,J-1\}\), we have
\(
0\le j_\m<J<p_\m .
\)
Therefore reducing \(K\) modulo \(p_\m\) recovers \(j_\m\) exactly:
\(
j_\m = K\bmod p_\m.
\)
Because \(p(\x)=p_\m\) on \(Q_\m\), 
the sub-network \(\Phi_2\) generates \(j_\m\) exactly.
Then through \(\Phi_3\), we can reconstruct the function value \(\Phi_3\circ\Phi_2\circ\Phi_1(\x) = \bar{f}(\x)\) from the quantization level. 

Now, we estimate the relationship between the error and the magnitude of the parameters. Since there are parameters in all $\Phi_1, \Phi_2, \Phi_3$ dependent on $\epsilon$, we provide their relationship, respectively.  

\underline{Parameters in $\Phi_1$}: By the local reconstruction estimate Eq.~\eqref{eq:main1-proof-idea-local-error},
\begin{equation}
|\bar{f}(\x)-f(\x)|\le\delta+A/M,\qquad\x\in\Omega .
\end{equation}
Since $f\in\operatorname{Lip}_A(\Omega)$, choosing $M\ge 2A/\varepsilon$ and $J\ge 4B/\varepsilon$ gives
$\delta=2B/J\le\varepsilon/2$. Hence,
\begin{equation}
\label{eq:error estimate}
|\bar{f}(\x)-f(\x)|\le A/M+\delta\le\varepsilon,\qquad\x\in\Omega .
\end{equation}
Fix $M_\varepsilon=\lceil 2A/\varepsilon\rceil$ and $J_\varepsilon=\lceil 4B/\varepsilon\rceil$. The
module $\Phi_1$ uses the grid scale $M_\varepsilon$ and the weights $1,M_\varepsilon,\ldots,M_\varepsilon^{D-1}$,
so
\begin{equation}
\mathcal P(\Phi_1)\le M_\varepsilon^D .
\end{equation}

\underline{Parameters in $\Phi_2$}:
The parameters of $\Phi_2$ are $c$ in Eq.~\eqref{beta} and the CRT integer $K$. The dominant one is
$K$, bounded by the modulus product: the least nonnegative solution satisfies
\begin{equation}
0\le K \le \prod_{i=1}^{M_\varepsilon^D}p_i -1
=\prod_{i=1}^{M_\varepsilon^D}\Bigl(1+i\,(J_\varepsilon+1)(M_\varepsilon^D)!\Bigr) -1.
\end{equation}
Thus 
\begin{equation}
    \mathcal P(\Phi_2)\le  \prod_{i=1}^{M_\varepsilon^D}\Bigl(1+i\,(J_\varepsilon+1)(M_\varepsilon^D)!\Bigr).
\end{equation}


We count the width and depth layer-wise over the entire composed network. In each layer, the width equals the number of neurons, including identity channels that carry a value forward for use in a later layer.
Following the standard approximation literature \citep{yarotsky2017error, shen2021floorrelu},
Each hidden layer consisting of affine transformations and element-wise activations is counted as one layer. 
The detailed calculation appears in Table~\ref{tab:width-depth-accounting}.

\underline{Width:}
In $\Phi_1$, the first layer's affine map scales the input to $M\x$, and its activation applies $\rho_{\floor}$ to produce $\lfloor M x_d \rfloor$ ($d=1,\dots,D$). Since no later module requires the original coordinates $x_d$, no carry channels are needed, giving this layer a width of $D$. 

In $\Phi_2^{(1)}$ and $\Phi_2^{(2)}$, the second layer's affine map linearly combines the $D$ floor outputs to directly compute the modulus $p(\mathbf{x})=1+c\,\Phi_1(\mathbf{x})$. Then, its activation applies $\rho_{\mathrm{inv}}$ to produce $1/p(\mathbf{x})$, while applying $\rho_1$ to carry $p(\mathbf{x})$ forward. This layer has the width of $2$.

In $\Phi_2^{(3)}$, the third layer's affine map scales $1/p(\mathbf{x})$ by the fixed target-dependent constant $K$ to obtain $K/p(\mathbf{x})$ while still carrying $p(\mathbf{x})$. Its activation applies $\rho_{\floor}$ to produce $q(\mathbf{x})=\lfloor K/p(\mathbf{x})\rfloor$ and $\rho_1$ to carry $p(\mathbf{x})$. This layer has the width of $2$.

In $\Phi_2^{(4)}$, the fourth layer constructs the components for $\mathtt{Mult}(p(\mathbf{x}),q(\mathbf{x}))$. Its affine map forms the four linear combinations $p+q$, $-p-q$, $p-q$, $q-p$ (width $4$). Then, its activation applies $\rho_2$ (RePU) to each component. Hence, this layer has width $4$.

In $\Phi_3$, the fifth (output) layer is a pure affine map. It first combines the $\rho_2$ outputs to evaluate $\Phi_2(\mathbf{x}) = K - \mathtt{Mult}(p,q)$, and simultaneously scales it by $\delta$ and subtracts $B$ to reconstruct $-B+j_{\mathbf{m}}\delta$. This outputs a scalar, giving a width of $1$.

The maximum width across all layers is therefore $\max\{D,\,4\}$.

\underline{Depth:}
By strictly grouping one affine map and one activation into a single layer, the network comprises $1$ layer for $\Phi_1$, $3$ layers for $\Phi_2$, and $1$ output layer for $\Phi_3$ (which absorbs the final affine summation of $\Phi_2$), giving a total depth of $5$.

\end{proof}

\begin{table}[htbp]
\centering
\footnotesize
\renewcommand{\arraystretch}{1.4}
\setlength{\tabcolsep}{5pt}
\begin{tabular}{c|c|p{0.42\textwidth}|c|c}
\toprule
Layer & Module & Operation & $\mathcal P_\ell$ bound & Width \\
\midrule
\rowcolor{phi11bg}
$1$ & $\Phi_1$
 & $\bigl(\rho_{\floor}(Mx_d)\bigr)_{d=1}^{D}$
 & $M_\varepsilon$
 &\cellcolor{widthHL}$D$ \\
\rowcolor{phi12bg}
$2$ & $\Phi_2^{(1)},\Phi_2^{(2)}$
 & $\bigl(\rho_{\inv}(p),\ \rho_1(p)\bigr)$
 & $1+c\,M_\varepsilon^{D-1}$
 & $2$ \\
\rowcolor{phi12bg}
$3$ & $\Phi_2^{(3)}$
 & $\bigl(\rho_{\floor}(\hp{K} \rho_{\inv}(p)),\ \rho_1(p)\bigr)$
 &  \cellcolor{paramHL}$\displaystyle\prod_{i=1}^{M_\varepsilon^{D}}p_i$
 & $2$ \\
\rowcolor{phi12bg}
$4$ & $\Phi_2^{(4)}$
 & $\bigl(\rho_2(p{+}q),\ \rho_2({-}p{-}q),\ \rho_2(p{-}q),\ \rho_2(q{-}p)\bigr)$
 & $1$
 &  \cellcolor{widthHL}$4$ \\
\rowcolor{phi13bg}
$5$ & $\Phi_2^{(4)},\Phi_3$
 & $-B+\delta\bigl(K-\mathtt{Mult}(p,q)\bigr)$ \
 & $\delta K+B$
 & $1$ \\
\bottomrule
\end{tabular}
\caption{Layer-wise  parameter magnitude and width-depth accounting for the continuous target construction.}
\label{tab:width-depth-accounting}
\end{table}


%% file: chapters/04_Theorem2.tex
\section{Approximation of H\"older-Smooth Functions}
\label{sec:section4}

In this part, we give a super-expressive approximation of H\"older-smooth targets, which demonstrates that smoother targets make parameters grow slower as the approximation error decreases. The key of the construction is that any H\"older-smooth targets can be approximated by a gridwise Taylor polynomial function of rational coefficients.

\begin{maintheorem}[H\"older-smooth function]
\label{thm:main-theorem2}
Let $f\in C^{r,\gamma}_A(\Omega)$ with $r\in\mathbb N$ and $\gamma\in(0,1]$. For every $\varepsilon\in(0,1)$, there exists a network $\Psi_{\varepsilon}$ with activation $\rho_{\floor}, \rho_{1}, \rho_{2}$ and $\rho_{\inv}$ as activation functions such that 
\begin{equation}
    \bigl|\Psi_{\varepsilon}(\x)-f(\x)\bigr|\le\varepsilon\qquad(\x\in\Omega),
\end{equation}
where the width is $\max\{2D,\, D+5N+1\}$ and the depth is $r+9$. Furthermore, the parameter magnitude of $\Psi_{\varepsilon}$ satisfies
\begin{equation}
\label{eq:corollary-holder-parameter-bound}
\log_2\mathcal P(\Psi_{\varepsilon})\le C_{D,r,\gamma,A,B}\,\varepsilon^{-2D/(r+\gamma)}\log_2(1/\varepsilon),
\end{equation}
with leading coefficient
\[
C_{D,r,\gamma,A,B}=2(1+\kappa)^{2D}\Bigl[N\bigl(2+2\log_2(2N)+\log_2(B+1)\bigr)+\tfrac{D}{r+\gamma}+D\log_2(1+\kappa)\Bigr],
\]
where
$$
\begin{aligned}
N &:= \binom{D+r}{r}, \qquad
\kappa := 
\left(
\frac{2D^{r}\Gamma(\gamma+1)}
     {\Gamma(r+\gamma+1)}
A
\right)^{1/(r+\gamma)}, \text{and}\\
B &:= 
\max_{0\le k\le r}\ 
\sup_{\x\in\Omega,\ g\in\partial^{k}f}
|g(\x)|.
\end{aligned}
$$
\end{maintheorem}

To give the proof, we first introduce two auxiliary lemmas. Lemma \ref{lem:lemma2} quantifies how the CRT parameter must grow as the accuracy increases, and Lemma \ref{lem:bitlength} guarantees that the Taylor coefficients can be rationalized to finite bit length without introducing approximation error.

\begin{lemma}[Gridwise polynomial with rational coefficients]
\label{lem:lemma2}
Let $\bigcup_{m=0}^{M^D-1}\Omega_{\m} $ be a uniform partition of the domain $[0,1]^D$, where $M$ is the number of intervals along each coordinate. Suppose $f:\Omega  \to\mathbb R$ is a gridwise polynomial with rational coefficients of degree $r$, then there exists a network $\Psi$ with activation $\rho_{\floor}$, $\rho_1$, $\rho_2$ and $\rho_{\inv}$, where the width is $\max\{2D,\, D+5N+1\}$ and the depth is $r+9$, such that
\begin{equation}
\label{eq:main-piecewise-polynomial-exact}
\Psi(\x)=f(\x).
\end{equation}
Furthermore, the parameter magnitude of $\Psi$ satisfies
\begin{equation}
\label{eq:main-piecewise-polynomial-parameter-bound}
\log_2 \mathcal P(\Psi)\le
M^D[(M^D+1)\,D \log_2 M\;+\;F\,(M^D N+1)\;+\;3] + 1,
\end{equation}
where $F$ is the maximum bit length of coefficients of $f$ in lowest terms, and $N := \binom{D+r}{r}$.
\end{lemma}

The detailed proof of Lemma \ref{lem:lemma2} can be found in Section \ref{sec:lemma2}.
Lemma~\ref{lem:lemma2} shows that rational gridwise polynomials can be exactly realized in a fixed architecture. This is significant for two reasons. First, the CRT provides an exact integer encoding of all gridwise data, with rationality ensuring that polynomial coefficients can be represented as integers. Second, the network recovers a full local polynomial on each grid, not just a constant label. This is achieved without growing the architecture with the target  information stored entirely in the CRT congruence data.

\begin{lemma}[Uniform low-bit rational approximation]
\label{lem:bitlength}
For every $\hat c\in\mathbb R$ with $|\hat c|\le L$ and $\forall 0<\eta< \frac{1}{2}$, there exists a rational number $c=s/g$ in lowest terms with $|\hat c - c| \le \eta$, satisfying
\begin{equation}
\operatorname{bits}(c)
\le
2\left\lceil \log_2\frac1\eta\right\rceil+\log_2(L+2)+2.
\end{equation}
\end{lemma}

\begin{proof}
Set
$
q:=\left\lceil \log_2\frac1\eta\right\rceil
$
and define the dyadic rounding
$
c:=2^{-q}\left\lfloor 2^q\hat c+\frac12\right\rfloor\in 2^{-q}\mathbb Z .
$
Since $\lfloor 2^q\hat c+\tfrac12\rfloor$ is a nearest integer to $2^q\hat c$, we have
$
\left|\left\lfloor 2^q\hat c+\frac12\right\rfloor-2^q\hat c\right|
\le \frac12 .
$
Therefore
$
|c-\hat c|
\le 2^{-q-1}
\le \eta .
$

If $\left\lfloor 2^q\hat c+\frac12\right\rfloor=0$, then $c=2^{-q}\left\lfloor 2^q\hat c+\frac12\right\rfloor=0$, and the claimed bounds are immediate. If $\left\lfloor 2^q\hat c+\frac12\right\rfloor\ne0$,  since every divisor of $2^q$ is a power of $2$, we may write $\gcd(|\left\lfloor 2^q\hat c+\frac12\right\rfloor|,2^q)=2^e$, where $0\le e\le q$. Hence, in lowest terms,
\begin{equation}
c=\frac{s}{g},
s=\frac{\left\lfloor 2^q\hat c+\frac12\right\rfloor}{2^e},
g=2^{q-e}.
\end{equation}
In particular,
$
g\le 2^q.
$
Moreover, using $|\hat c|\le L$ and the rounding estimate above,
\begin{equation}
|s|=|c|g
\le \bigl(|\hat c|+2^{-q-1}\bigr)2^q
\le \left(L+\frac12\right)2^q
\le (L+1)2^q .
\end{equation}
Therefore,
\begin{equation}
|s|+2\le (L+1)2^q+2\le (L+2)2^q,
\end{equation}
where we use $q\ge1$. Also,
\begin{equation}
    g+1\le 2^q+1\le 2^{q+1}.
\end{equation}
By the definition of bit length,
$
\left\lceil \log_2(|s|+2)\right\rceil
\le
\log_2(L+2)+q+1,
$
and
$
\left\lceil \log_2(g+1)\right\rceil
\le q+1.
$
Adding two estimates gives
\begin{equation}
\operatorname{bits}(c)
\le
2q+\log_2(L+2)+2
=
2\left\lceil \log_2\frac1\eta\right\rceil+\log_2(L+2)+2.
\end{equation}
This bound depends only on $L$ and $\eta$ instead of the particular value of $\hat c$.
\end{proof}

\begin{proof}[Proof of Theorem~\ref{thm:main-theorem2}]
Fix $\varepsilon\in(0,1)$ and an integer $M\ge1$ to be chosen in Step~1 in Theorem \ref{thm:main-theorem1}, and partition $\Omega=[0,1]^D$ into $M^D$ grids $\Omega_{\m}$ with corners $\x_{\m}=\m/M$, so that on $\Omega_{\m}$ the local coordinate is $M(\x-\x_{\m})\in[0,1]^D$. We construct a rational gridwise  polynomial $\tilde f$ of degree $r$ with $\|f-\tilde f\|_{L^\infty(\Omega)}\le\varepsilon$ and realize it exactly via Lemma~\ref{lem:lemma2}.

\underline{Taylor error estimate.}
Let
\begin{equation}
T_{\m}(\x):=\sum_{|\balpha|\le r}\frac{\partial^{\balpha}f(\x_{\m})}{\balpha!}(\x-\x_{\m})^{\balpha}
=\sum_{i\le N}\hat c_{\m,i}\,\bigl(M(\x-\x_{\m})\bigr)^{\balpha^{(i)}},
\end{equation}
and
\begin{equation}
c_{\m,i}:=\frac{\partial^{\balpha^{(i)}}f(\x_{\m})}{\balpha^{(i)}!\,M^{|\balpha^{(i)}|}}
\end{equation}
be the degree-$r$ Taylor polynomial of $f$ at $\x_{\m}$, using $(\x-\x_{\m})^{\balpha}=M^{-|\balpha|}\bigl(M(\x-\x_{\m})\bigr)^{\balpha}$. For $r\ge1$, put $\phi(t):=f\bigl(\x_{\m}+t(\x-\x_{\m})\bigr)$. Taylor's theorem with integral remainder, with the order-$r$ term moved into $T_{\m}$ via $\tfrac{1}{(r-1)!}\int_0^1(1-t)^{r-1}\,dt=\tfrac1{r!}$, gives
\begin{equation}
f(\x)-T_{\m}(\x)
=\frac{1}{(r-1)!}\int_0^1(1-t)^{r-1}\bigl(\phi^{(r)}(t)-\phi^{(r)}(0)\bigr)\,dt .
\end{equation}
Since $\phi^{(r)}(t)=\sum_{|\balpha|=r}\tfrac{r!}{\balpha!}(\x-\x_{\m})^{\balpha}\,\partial^{\balpha}f\bigl(\x_{\m}+t(\x-\x_{\m})\bigr)$ and $\|\x-\x_{\m}\|_\infty\le1/M$,
\begin{equation}
\bigl|\phi^{(r)}(t)-\phi^{(r)}(0)\bigr|
\le r!\,A\,t^{\gamma}\|\x-\x_{\m}\|_\infty^{\gamma}\sum_{|\balpha|=r}\frac{|(\x-\x_{\m})^{\balpha}|}{\balpha!}
\le r!\,A\,t^{\gamma}M^{-\gamma}\cdot\frac{(D/M)^{r}}{r!},
\end{equation}
where the last step uses $\sum_{|\balpha|=r}\tfrac{|(\x-\x_{\m})^{\balpha}|}{\balpha!}=\tfrac1{r!}\bigl(\sum_d|x_d-x_{\m,d}|\bigr)^{r}\le\tfrac{(D/M)^{r}}{r!}$. With $\int_0^1(1-t)^{r-1}t^{\gamma}\,dt=\tfrac{\Gamma(\gamma+1)\Gamma(r)}{\Gamma(r+\gamma+1)}$,
\begin{equation}
|f(\x)-T_{\m}(\x)|
\le\frac{D^{r}\Gamma(\gamma+1)}{\Gamma(r+\gamma+1)}\,A\,M^{-(r+\gamma)}
=C_{D,r,\gamma}\,A\,M^{-(r+\gamma)} .
\end{equation}
For $r=0$, this is the direct H\"older bound $|f(\x)-f(\x_{\m})|\le A\|\x-\x_{\m}\|_\infty^{\gamma}\le AM^{-\gamma}$, again with $C_{D,0,\gamma}=1$. Choosing
\begin{equation}
M=M_\varepsilon:=\Bigl\lceil(2C_{D,r,\gamma}A)^{1/(r+\gamma)}\,\varepsilon^{-1/(r+\gamma)}\Bigr\rceil
\end{equation}
yields $|f-T_{\m}|\le\varepsilon/2$ on every grid.

\underline{Rationalization.}
Since $|\partial^{\balpha^{(i)}}f(\x_{\m})|\le B$ and $\balpha^{(i)}!,M^{|\balpha^{(i)}|}\ge1$, we have $|\hat c_{\m,i}|\le B$. Apply Lemma~\ref{lem:bitlength} with $L=B$ and $\eta:=\varepsilon/(2N)<\tfrac12$ (valid as $\varepsilon<1\le N$): each $\hat c_{\m,i}$ admits a rational $c_{\m,i}=s_{\m,i}/g_{\m,i}$ in lowest terms with $|c_{\m,i}-\hat c_{\m,i}|\le\eta$ and
\begin{equation}
\mathtt{bits}(c_{\m,i})\le 2\bigl\lceil\log_2(2N/\varepsilon)\bigr\rceil+\log_2(B+2)+2=:F_\varepsilon .
\end{equation}
Define $\tilde f(\x):=\sum_{i\le N}c_{\m,i}\,\bigl(M(\x-\x_{\m})\bigr)^{\balpha^{(i)}}$ on $\Omega_{\m}$. As $0\le\bigl(M(\x-\x_{\m})\bigr)^{\balpha^{(i)}}\le1$,
\begin{equation}
|T_{\m}(\x)-\tilde f(\x)|\le\sum_{i\le N}|\hat c_{\m,i}-c_{\m,i}|\le N\eta=\varepsilon/2,
\end{equation}
hence $|f(\x)-\tilde f(\x)|\le\varepsilon$ for all $\x\in\Omega$.

\underline{Realization by a fixed-size network.}
$\tilde f$ is a gridwise polynomial with rational coefficients of degree $r$ on the uniform $M_\varepsilon^{D}$-partition, so Lemma~\ref{lem:lemma2} furnishes a network $\Psi_\varepsilon$ of the same activations, width $\max\{2D,\,D+5N+1\}$, and depth $r+9$ with $\Psi_\varepsilon=\tilde f$. The architecture depends only on $D,r$ and $N$; hence, it is independent of $\varepsilon$ and $f$, and $|\Psi_\varepsilon(\x)-f(\x)|\le\varepsilon$.

\underline{Parameter magnitude.}
By Lemma~\ref{lem:lemma2} with $M=M_\varepsilon$ and bit length $F=F_\varepsilon$ (for $r\ge1$, $r=0$ follows from the exact bound of Lemma~\ref{lem:lemma2} with $N=1$),
\begin{equation}
\log_2\mathcal P(\Psi_\varepsilon)\le 2M_\varepsilon^{2D}\bigl(D\log_2 M_\varepsilon+N F_\varepsilon\bigr).
\end{equation}
Write $\kappa:=\bigl(\tfrac{2D^{r}\Gamma(\gamma+1)}{\Gamma(r+\gamma+1)}A\bigr)^{1/(r+\gamma)}$. The choice of $M_\varepsilon$ and $\varepsilon<1$ give
\begin{equation}
M_\varepsilon^{2D}\le(1+\kappa)^{2D}\varepsilon^{-2D/(r+\gamma)},\qquad
\log_2 M_\varepsilon\le\log_2(1+\kappa)+\tfrac{1}{r+\gamma}\log_2(1/\varepsilon),
\end{equation}
while $F_\varepsilon\le 2\log_2(1/\varepsilon)+2\log_2(2N)+\log_2(B+1)+O(1)$. Substituting and collecting the coefficient of $\varepsilon^{-2D/(r+\gamma)}\log_2(1/\varepsilon)$ yields
\begin{equation}
\log_2\mathcal P(\Psi_\varepsilon)\le C_{D,r,\gamma,A,B}\,\varepsilon^{-2D/(r+\gamma)}\log_2(1/\varepsilon),
\end{equation}
with
\begin{equation}
C_{D,r,\gamma,A,B}=2(1+\kappa)^{2D}\Bigl[N\bigl(2+2\log_2(2N)+\log_2(B+1)\bigr)+\tfrac{D}{r+\gamma}+D\log_2(1+\kappa)\Bigr],
\end{equation}
which is the constant in the statement.
\end{proof}


\begin{figure}[htbp]
    \centering
    \includegraphics[width=0.8\textwidth]{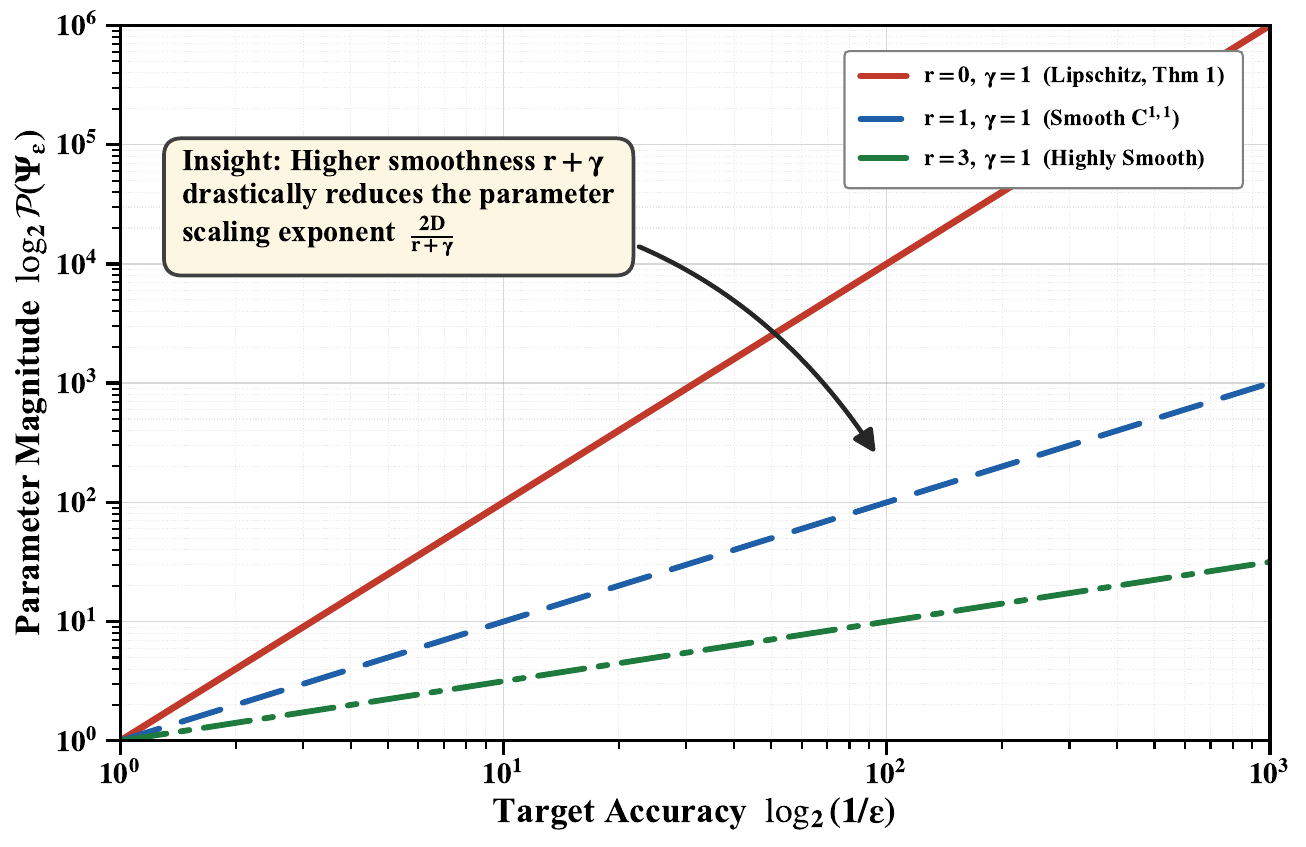}
    \vspace{-5mm}
    \caption{Scaling law of the parameter magnitude $\log_2 \mathcal{P}(\Psi_\varepsilon)$ with respect to the target accuracy $\log_2(1/\varepsilon)$.}
    \label{fig:scaling_law}
\end{figure}

\section{Proof of Lemma \ref{lem:lemma2}}
\label{sec:lemma2}

Lemma \ref{lem:lemma2} shows an exact realization for gridwise polynomials with rational coefficients.
Exact realization phenomena already appear in several related settings, but with different scopes. The exact result of \citep{shen2022fixed} concerns finite-valued classification functions on pairwise disjoint bounded closed sets.  Its key mechanism is to directly leverage the distance function 
\begin{equation}
f(\x):=
\frac{\operatorname{dist}(\x,B)}
{\operatorname{dist}(\x,A)+\operatorname{dist}(\x,B)}
\qquad
\text{for any } \x\in\mathbb{R}^D,
\end{equation}
where $A$ and $B$ are two disjoint bounded closed sets, and eliminate the error at the margin: the labels are rescaled to separated odd integers, a continuous lifting is approximated within less than half the label gap, and a final correction map sends the approximate value exactly back to the intended label. But in \citep{shen2022fixed}, the distance function is not represented by a network; therefore, its form is implicit. The ReQU-based exact algebraic result of \citep{li2020powernet} is best understood as an exact polynomial-evaluation theorem.  Since a ReQU activation $\rho_s(t)$ directly generates power functions, products and monomials can be built exactly by algebraic identities and network composition. Linear combinations of these monomials then represent a polynomial with no approximation error.  In this sense, once the polynomial coefficients are known, ReQU networks provide a natural exact evaluator for local polynomial rules.  The scope of this exactness is algebraic: it applies to polynomials and to smooth functions only after polynomial approximants have first replaced them. 

\subsection{Proof Idea of Lemma \ref{lem:lemma2}}
\label{subsec:proof-idea-thm2}

The proof is parallel to that of \Cref{thm:main-theorem1}, with one important difference.  The recovered gridwise object is a quantized constant, whereas here the recovered object is a local polynomial.  The exact construction is feasible because, on each grid, the local polynomial is exactly determined by finitely many rational coefficients, and these coefficients can be represented exactly by integers.

\emph{1. Calculation of the local coordinates.} For gridwise polynomials, we do not need to partition their input domain, as they are already defined on uniform grids. Thus, as in Steps~1--2 of Section~\ref{subsec:proof-maintheorem1}, a single floor layer computes $\rho_\floor(Mx_d)$, $d=1,\dots,D$, from which we read off two quantities. The first is the grid annotation
\begin{equation}
\label{eq:calculate-tx}
\Psi^{(1)}_1(\x):=1+\sum_{d=1}^D M^{d-1}\rho_\floor(Mx_d)\in\{1,\dots,M^D\},
\end{equation}
the base-$M$ encoding of the active grid. Because here our goal is an exact representation of local polynomials, we do not need to quantize $f$. Instead, we compute the local coordinate, 
\begin{equation}
\Psi^{(2)}_1(\x):=M(\x-\x_{\m})=M\x-\rho_\floor(M\x)\in[0,1)^D,
\end{equation}
where  $\x_\m:=\m/M$ is the corner of $\Omega_\m$.
Sharing the same $D$ floor units and carrying $\x$ forward, the sub-network is
\begin{equation}
\Psi_1(\x):=\big(\Psi^{(1)}_1(\x),\,\Psi^{(2)}_1(\x)\big).
\end{equation}

\emph{2. Exact representation by coordinate-wise CRT.}
Setting $N:=\binom{D+r}{r}$, we fix the ordering of the polynomial exponents
$
\{\balpha\in\mathbb N^D: 0\le|\balpha|\le r\}
=\{\balpha^{(1)},\dots,\balpha^{(N)}\}.
$
Each grid $\Omega_{\m}$ carries a rational polynomial that, in the local coordinate $M(\x - \x_{\m})\in[0,1]^D$, reads
\begin{equation}
Y_\m\!\big(\x - \x_\m \big)
=\sum_{i\le N} c_{\m,i}\,M^{|\balpha^{(i)}|}(\x - \x_\m)^{\balpha^{(i)}} ,
\qquad c_{\m,i}\in\mathbb Q.
\end{equation}
Since the index set $\{(\m,i):1\le |\m| \le M^D,\ 1\le i\le N\}$ is finite, so is the set $\{c_{\m,i}\}$. Since every $c_{\m,i}$ is a rational, we can multiply every $c_{\m,i}$ with a common integer such that we have $c_{\m,i}=s_{\m,i}/B_f$, and $s_{\m,i}\in\mathbb Z$. Thus, all our later constructions can be built on integers. With $C_f:=\max_{\m,i}|s_{\m,i}|$, we define
\begin{equation}
\Gamma_f:=(1+2C_f)\,(M^D)! .
\end{equation}
The sub-network $\Psi^{(1)}_{2}$ computes the modulus by one affine layer,
\begin{equation}
\Psi^{(1)}_{2}(\x):=1+\Psi^{(1)}_1(\x) \Gamma_f = p_{\m} .
\end{equation}
Because we use the same construction, by Lemma~\ref{lem:godel-crt-moduli-coprime}, $p_{\m}>\Gamma_f$ and $\{p_{\m}\}_{|\m|=1}^{M^D}$ are pairwise coprime. For each monomial index $i\le N$, Proposition~\ref{prop:crt} gives an integer $K_i$ with
\begin{equation}
\label{eq:Ki}
K_i\equiv s_{\m,i}\pmod{p_{\m}},\qquad m=1,\dots,M^D .
\end{equation}
Since $s_{\m,i}\in[-C_f,C_f]$, setting $E_{\m,i}:=K_i\bmod p_{\m}$, Lemma~\ref{lem:stlRange} gives the signed representative
\begin{equation}
s_{\m,i}=
\begin{cases}
E_{\m,i}, & 0\le E_{\m,i}\le C_f,\\[2pt]
E_{\m,i}-p_{\m}, & p_{\m}-C_f\le E_{\m,i}<p_{\m} .
\end{cases}
\end{equation}
The residue $E_{\m,i}$ is realized by $\Psi^{(2)}_{2,i}$:
\begin{equation}
\label{eq:Phi22}
\Psi^{(2)}_{2,i}(\x):=E_{\m,i}
=K_i-\Psi^{(1)}_{2}(\x)\,\rho_{\floor}\!\big(K_i ~\rho_{\inv}(\Psi^{(1)}_{2}(\x))\big).
\end{equation}
By Lemma~\ref{lem:recover-Sti}, the signed value $s_{\m,i}$ can be recovered from $\Psi^{(3)}_{2,i}$ by one $\rho_{\inv}$, one $\rho_{\floor}$, and the multiplication module
$\mathtt{Mult}(u,v):=\big(\rho_2(u+v)+\rho_2(-u-v)-\rho_2(u-v)-\rho_2(v-u)\big)/4$:
\begin{equation}
\Psi^{(3)}_{2,i}(\x):=s_{\m,i}
=\Psi^{(2)}_{2,i}(\x)-\mathtt{Mult}\!\Big(\Psi^{(1)}_{2}(\x),\,
\rho_{\floor}\big(\mathtt{Mult}(\Psi^{(2)}_{2,i}(\x),\rho_{\inv}(\Psi^{(1)}_{2}(\x) - C_f))\big)\Big),
\end{equation}
so that $c_{\m,i}=B^{-1}_f s_{\m,i}=B^{-1}_f\Psi^{(3)}_{2,i}(\x)$.

For each fixed $i$, the coefficient $c_{\m,i}$ is produced by the serial chain $\Psi^{(3)}_{2,i}\circ\Psi^{(2)}_{2,i}$ built on the shared modulus $\Psi^{(1)}_{2}$. These $N$ chains run over the single modulus channel $\Psi^{(1)}_{2}(\x)$ in parallel, while the local coordinate $\Psi^{(2)}_{1}(\x)$ from Step~1 is fed into later layers unchanged for Step~3. Collecting the outputs, $\Psi_2$ sends the pair $\big(\Psi^{(1)}_1(\x),\Psi^{(2)}_{1}(\x)\big)$ to
\begin{equation}
\Psi_2(\x):=\Big(B^{-1}_f\Psi^{(3)}_{2,i}(\x)_{i\le N},\ \Psi^{(2)}_{1}(\x)\,\Big)
=\big(\c_{\m},\ \Psi^{(2)}_{1}(\x)\big),
 \c_{\m}:=(c_{\m,1},\dots,c_{\m,N}).
\end{equation}

\emph{3. Exact polynomial computation.}
From the local coordinate $\Psi^{(2)}_{1}(\x)\in[0,1]^D$, define for each $i\le N$ the monomial
\begin{equation}
\label{Phi31}
\Psi^{(1)}_{3,i}(\x):=[\Psi^{(2)}_{1}(\x)]^{\balpha^{(i)}}
=\prod_{d=1}^{D}\big(\Psi^{(2)}_{1,d}(\x)\big)^{\alpha^{(i)}_d}.
\end{equation}
By Lemma~\ref{lem:ReQU-Multiplication}, each monomial $[\Psi^{(2)}_{1}(\x)]^{\balpha^{(i)}}$ is computed exactly. 
Running the $N$ blocks in parallel on disjoint channels (sharing the input $\Psi^{(2)}_1(\x)$ and carrying the coefficients $c_{\m,i}$), and aligning every chain to terminate at layer $r$, yields all monomials simultaneously. We write
$\Psi^{(1)}_{3}:=(\Psi^{(1)}_{3,1},\dots,\Psi^{(1)}_{3,N})$. The module $\Psi^{(2)}_{3}$ then scales each monomial by its recovered coefficient via $\mathtt{Mult}$ and sums by one affine layer:
\begin{equation}
\Psi^{(2)}_{3}\!\big(\c_{\m},\Psi^{(1)}_{3}(\x)\big)
:=\sum_{i\le N}\mathtt{Mult}\!\big(c_{\m,i},\,\Psi^{(2)}_{1}(\x)^{\balpha^{(i)}}\big)
=Y_{\m}\!\big(\Psi^{(2)}_{1}(\x)\big).
\end{equation}
and we set $\Psi_3\big(\c_{\m},\Psi^{(2)}_1(\x)\big):=\Psi^{(2)}_3\big(\c_{\m},\Psi^{(1)}_3(\Psi^{(2)}_1(\x))\big)$.

Each module is thus a map on the copied tuple:
\begin{equation}
\Psi_1:\x\mapsto\big(\Psi^{(1)}_1(\x),\Psi^{(2)}_{1}(\x)\big); 
\Psi_2:\big(\Psi^{(1)}_1,\Psi^{(2)}_{1}\big)\mapsto\big(\c_{\m},\Psi^{(2)}_{1}\big); 
\Psi_3:\big(\c_{\m},\Psi^{(2)}_{1}\big)\mapsto Y_{\m}\!\big(\Psi^{(2)}_{1}\big),
\end{equation}
so the three compose serially. Since $f(\x)=Y_{\m}(\x - \x_{\m})$ on $\Omega_{\m}$,
\begin{equation}
\Psi(\x):=\Psi_3\circ\Psi_2\circ\Psi_1(\x)=Y_{\m}(\x - \x_{\m})=f(\x).
\end{equation}

The block diagram of the network $\Psi = \Psi_3\circ\Psi_2\circ\Psi_1(\x)$  is visualized in Figure \ref{fig:thm2_arch}.
\begin{figure}[htbp]
    \centering
      \includegraphics[width=0.95\textwidth]{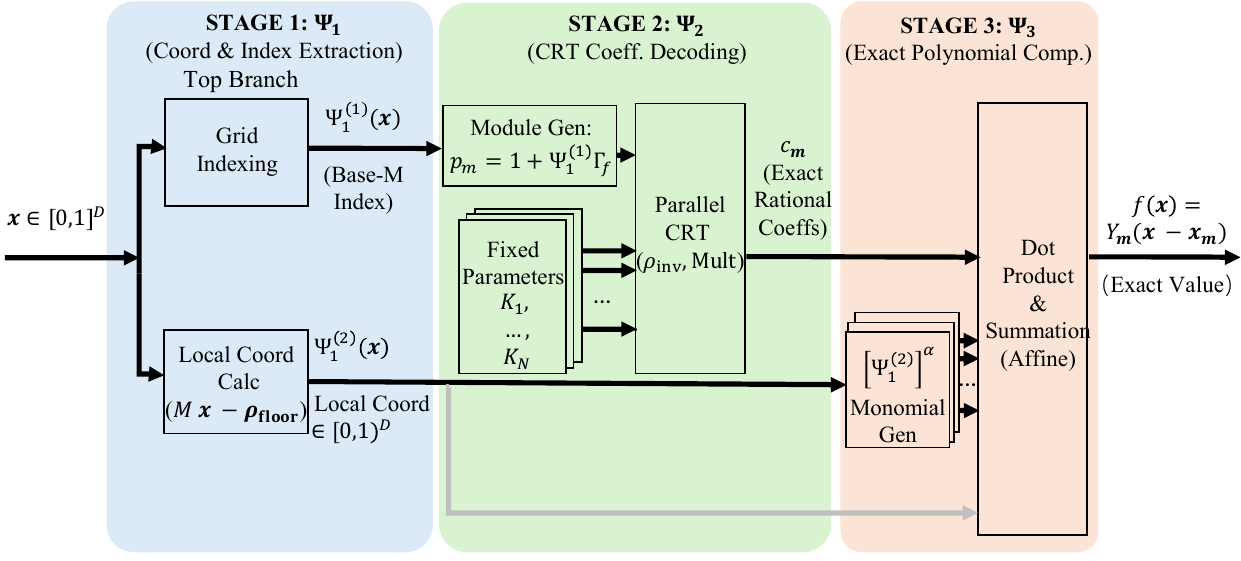}
    \caption{Architecture of the target network realizing $\Psi$.}
    \label{fig:thm2_arch}
\end{figure}

\subsection{Proof of Lemma \ref{lem:lemma2}}

\begin{lemma}[Range of the canonical residue]
\label{lem:stlRange}
Let $C,p\in\mathbb N_+$ with $p>2C$, and let $s\in[-C,C]\cap\mathbb Z$ and $K\in\mathbb Z$ satisfy $K\equiv s\pmod p$. 
Then $K\bmod p$ lies in one of the disjoint intervals $[0,C]$ and $[p-C,p)$.
\end{lemma}

\begin{proof}
Since $K\equiv s\pmod p$ is the unique element which is congruent to $s$ modulo $p$.

\begin{enumerate}
    \item If $s\ge 0$, then $0\le s\le C<p$, so $s$ is already in $\{0,\dots,p-1\}$ and
$s = K\bmod p\in[0,C]$. 
 \item  If $s<0$, then $-C\le s<0$, so $p-C\le p+s<p$ with $p+s\in\{0,\dots,p-1\}$. Hence $K\bmod p=p+s\in[p-C,p)$ and $s=K\bmod p-p$.
\end{enumerate}
Because $p>2C$ gives $C<p-C$,  the intervals $[0,C]$ and $[p-C,p)$ are disjoint, and $K\bmod p$ falls in one of them. 
\end{proof}

\begin{lemma}[Recovering $s$ via $\rho_{\inv}$, $\rho_{\floor}$ and $\rho_2$]
\label{lem:recover-Sti}
Given
$s:\big([0,C]\cup[p-C,p)\big)\times\mathbb N_+\to\mathbb Z$ by
\begin{equation}
\label{eq:def-signed-s}
s(E,p)=
\begin{cases}
E, & 0\le E\le C,\\[2pt]
E-p, & p-C\le E<p,
\end{cases}
\end{equation}
where $C, p\in\mathbb N_+$ with $p>2C$, so that the two intervals $[0,C]$ and $[p-C,p)$ are disjoint, there is a network with activation $\rho_{\floor}$, $\rho_{\inv}$ and $\rho_2$ to realize it.
\end{lemma}

\begin{proof} There are two cases:
\begin{enumerate}
    \item If $E \in [0,C]$, then $0\le \frac{E}{p-C}<1$. Therefore $\rho_{\floor}\left(\frac{E}{p-C}\right) = 0$. 
    \item If  $E \in [p-C,p)$, then $1\le \frac{E}{p-C}<\frac{p}{p-C}<2$. Therefore $\rho_{\floor}\left(\frac{E}{p-C}\right) = 1$. 
\end{enumerate}
So $s =E-p ~ \rho_{\floor}\left(\frac{E}{p-C}\right).$ 
The quantity $(p-C)^{-1}$ can be produced by $\rho_{\inv}$, the product can be produced by  $\mathtt{Mult}$ based on $\rho_2$, and cases are distinguished with $\rho_{\floor}$. Hence $s(E,p)$ is realized exactly on $[0,C]\cup[p-C,p)$.
\end{proof}

\begin{lemma}[ReQU realization of a monomial]
\label{lem:ReQU-Multiplication}
Given a monomial
\begin{equation}
\label{eq:monomial}
\x^{\balpha}=\prod_{d=1}^{D}x_d^{\alpha_d},\qquad \x\in[0,1]^D,
\end{equation}
where $\balpha=(\alpha_1,\dots,\alpha_D)\in\mathbb N^D$ is a multi-index of degree
$|\balpha|=\sum_{d=1}^D\alpha_d\le r$, there exists a network with activation $\rho_2$
(through $\mathtt{Mult}$) that computes $\x^{\balpha}$ exactly on $[0,1]^D$, with width $4$ and depth $|\balpha|$.
\end{lemma}

\begin{proof}
By Eq. \eqref{eq:multiple}, $\mathtt{Mult}(u,v)=uv$, using four ReQU units in a single layer.
Each coordinate $x_d$ is repeated  $\alpha_d$ times, writing as $\x^{\balpha}$.
Accumulate one factor per layer, as:
\begin{equation}
\x^{\balpha}
=\mathtt{Mult}\Bigl(\cdots\mathtt{Mult}\bigl(\mathtt{Mult}(1,
\underbrace{x_1,\dots,x_1}_{\alpha_1}),\,
\underbrace{x_2,\dots,x_2}_{\alpha_2}\bigr),\dots,\,
\underbrace{x_D,\dots,x_D}_{\alpha_D}\Bigr).
\end{equation}
Each $\mathtt{Mult}$ is one layer of four ReQU units, so the width is $4$; there are $|\balpha|$ such layers, so the depth is $|\balpha|$.
\end{proof}




\begin{proof}[Proof of Lemma~\ref{lem:lemma2}]
The construction of Steps~1--3 produces $\Psi=\Psi_3\circ\Psi_2\circ\Psi_1$.  Specially, $\Psi_1$ outputs the grid index $\Psi^{(1)}_1(\x)$ together with the local coordinate $M(\x-\x_{\m})\in[0,1)^D$. Next, 
by Lemma~\ref{lem:godel-crt-moduli-coprime} the moduli $\{p_{\m}\}$ over the $M^D$ grids are pairwise coprime, so Proposition~\ref{prop:crt} furnishes integers $K_i$ with $K_i\equiv s_{\m,i}\pmod{p_{\m}}$ for every grid.
Then  Lemma~\ref{lem:stlRange} gives $E_{\m,i}\in[0,C_f]\cup[p_{\m}-C_f,p_{\m})$, and Lemma~\ref{lem:recover-Sti} recovers $s_{\m,i}$ from $E_{\m,i}$ exactly through $\rho_{\inv},\rho_{\floor},\mathtt{Mult}$. Thus, $c_{\m,i}=B_f^{-1}s_{\m,i}$.
So that $\Psi_2$ recovers the exact rational coefficients $(c_{\m,i})_{i\le N}$ of the polynomial $Y_{\m}$ on the active grid $\Omega_{\m}$.
Finally, $\Psi_3$ constructs the target polynomial: by Lemma~\ref{lem:ReQU-Multiplication}, each monomial $(\x-\x_\m)^{\balpha^{(i)}}$ is realized exactly, so is $\sum_{i\le N}c_{\m,i}\,(\x-\x_\m)^{\balpha^{(i)}}$. Since our network can exactly represent every polynomial over each grid, our network can represent exactly the whole polynomial.

It remains to bound the width, depth, and parameter magnitudes, with $N:=\binom{D+r}{r}$ and $M^D$ grids. As Table~\ref{tab:width-depth-poly} shows, our counting follows the standard approximation literature: the width of a layer is its number of neurons, including identity channels copied forward; each hidden layer consisting of an affine map followed by element-wise activations counts as one layer.

\medskip
\underline{Width and depth of $\Psi_1$.}
In $\Psi_1$, it first applies $\rho_{\floor}$ to each $Mx_d$ and carries $\x$ (needed to form the local coordinate) as the input. Then, the affine combination for the index $\Psi^{(1)}_1(\x)$ together with $\Psi^{(2)}_1(\x)=M\x-\rho_{\floor}(M\x)$. Hence, for $\Phi_1$, we have $W_1=2D$ and $L_1=1$.

\medskip
\underline{Width and depth of $\Psi_2$.}
In $\Psi_2$, $D$ channels in $\Psi^{(2)}_1(\x)$ are copied until being passed into $\Psi_3$. Recovering coefficients uses two reciprocals: $1/p_{\m}$ for the residue $E_{\m,i}=K_i-p_{\m}\lfloor K_i/p_{\m}\rfloor$, and $1/(p_{\m}-C_f)$ for the signed value $s_{\m,i}=E_{\m,i}-p_{\m}\lfloor E_{\m,i}/(p_{\m}-C_f)\rfloor$. Forming $p_{\m}$ and $p_{\m}-C_f$ together, a single $\rho_{\inv}$ layer produces both. This yields Layers~2-7, as seen in Table~\ref{tab:width-depth-poly}, which results in a depth of 6. The widest layer is $\mathtt{Mult}$, which forms $\mathtt{Mult}(E{\m,i},1/(p_{\m}-C_f))$ for all $i\le N$, with a width of $D+5N+1$.

\medskip
\underline{Width and depth of $\Psi_3$.}
By Lemma~\ref{lem:ReQU-Multiplication}, each monomial $[\Psi^{(2)}_1(\x)]^{\balpha^{(i)}}$ is a sequential product of $|\balpha^{(i)}|\le r$ factors, one $\mathtt{Mult}$ per Layer. Running $N$ accumulators in parallel and padding shorter chains with $\mathtt{Mult}(\cdot,1)$, all monomials are obtained after multiplying the entries of $\Psi^{(2)}_1(\x)$, with the coefficient extraction $c_{\m,i}$ sharing the first of these layers. One further $\mathtt{Mult}$ layer scales each monomial by its coefficient, and one affine map sums them to $Y_\m(\x-\x_\m)$. The widest layer runs $N$ multiplications, which needs $4N$ ReQU units with $D+N$ copied channels. Hence, for $\Psi_3$, we have $W_3=D+5N$ and $L_3=r+2$. 

Now we conclude the total width and depth  based on the width and depth of three modules:
\begin{equation}
\begin{cases}
    & W(\Psi)=\max\{W_1,W_2,W_3\}=\max\{2D,\ D+5N+1\},\\
& L(\Psi)=L_1+L_2+L_3=r+9 ,
\end{cases}
\end{equation}
where $W(\Psi)=D+5N+1$ for $r\neq 0$.

\medskip
\underline{Parameter magnitudes.} We now estimate the largest parameter magnitude $\mathcal P(\Psi)$.
The parameters of $\Psi_1$ are $M,M^2,\dots,M^{D-1}$; those of $\Psi_2$ are $\Gamma_f$, $C_f$, $B_f^{-1}$, $K_i$, and fixed $\mathtt{Mult}$ weights of absolute value $\le1$; and those of $\Psi_3$ are padding constants $1$ and $\mathtt{Mult}$ weights $\le1$. The dominant parameters are the CRT integers $K_i$, controlled by the modulus product. Each modulus $p_{\m}=1+\Psi^{(1)}_1(\x)\Gamma_f$ with $\Psi^{(1)}_1(\x)\le M^D$ satisfies $p_{\m}\le p_{\max}:=1+M^D\Gamma_f$ and
$\Gamma_f=(1+2C_f)(M^D)!$. Since the least nonnegative CRT solutions obey $0\le K_i<\prod_{\m}p_{\m}\le p_{\max}^{M^D}$, and all other parameters are at most $C_f\le p_{\max}^{M^D}$, we have
\begin{equation}
\label{eq:param-bound-abstract}
\mathcal P(\Psi)\le p_{\max}^{M^D}+C_f\le 2\bigl(1+M^D(1+2C_f)(M^D)!\bigr)^{M^D},
\end{equation}
which is an explicit bound with $M,D,N$ and the integer scale $C_f$ of the rational coefficients.

The bound in Eq. \eqref{eq:param-bound-abstract} depends on $f$ only through the integer scale $C_f$ of its coefficients. For any fixed gridwise polynomial with rational coefficients, $C_f$ is finite and determined by $f$. Remember we write each $c_{\m,i}=s_{\m,i}/g_{\m,i}$ in lowest terms, let
\begin{equation}
F:=\max_{\m,i}\mathtt{bits}(c_{\m,i}),
\label{bound}
\end{equation}
and the coefficient bit length carried by $f$. Then $|s_{\m,i}|<2^F$ and $g_{\m,i}<2^F$, so $B_f\le 2^{F M^D N}$ and $C_f\le 2^F B_f\le 2^{F(M^D N+1)}$. Substituting Eq. \eqref{bound} into Eq. \eqref{eq:param-bound-abstract} and combining $(M^D)!\le(M^D)^{M^D}$, $1+2C_f\le 2^{F(M^D N+1)+2}$, and $1+M^D\Gamma_f\le 2M^D\Gamma_f$ give
\begin{equation}
\log_2 p_{\max}\le D(M^D+1)\log_2 M+F(M^D N+1)+3.
\end{equation}
 Therefore
\begin{equation}
\label{eq:param-bound-thm2}
\mathcal P(\Psi)\le
2^{\,M^D[(M^D+1)\,D \log_2 M\;+\;F\,(M^D N+1)\;+\;3] + 1},
\end{equation}
with $\mathcal P(\Psi)\le 2 ^{2 M^{2D}(D \log_2 M + FN)}$ for $r \neq 0$.


\end{proof}

\begin{table}[htbp]
\centering
\footnotesize
\renewcommand{\arraystretch}{1.5}\small
\setlength{\tabcolsep}{4pt}
\begin{tabular}{c|c|p{0.4\textwidth}|p{0.1\textwidth}|c}
\toprule
Layer & Module & Operation & $\mathcal P_\ell$ bound & Width \\
\midrule
\rowcolor{phi1bg}
$1$ & $\Psi^{(1)}_1$
 & $\bigl(\rho_\floor(Mx_d),\,\rho_1(Mx_d)\bigr)_{d\le D}$
 & $M$
 & \cellcolor{widthHL}$D{+}D$ \\
\rowcolor{phi2bg}
$2$ & $\Psi^{(2)}_1,\Psi^{(1)}_2$
 & $\bigl(\rho_\inv(p_\m),\,\rho_\inv(p_\m{-}C_f)\,,\,p_\m,\,\Psi^{(2)}_1\bigr)$
 & $\Gamma_f M^{D-1}$
 & $1{+}1{+}1{+}D$ \\
\rowcolor{phi2bg}
$3$ & $\Psi^{(2)}_2$
 & $\bigl(\rho_\floor(\hp{K_i}/p_\m)_{i\le N}\,,\,r_\m,\,p_\m,\,\Psi^{(2)}_1\bigr)$
 & \cellcolor{paramHL}$\displaystyle\prod_{\m}p_\m$
 & $N{+}1{+}1{+}D$ \\
\rowcolor{phi2bg}
$4$ & $\Psi^{(2)}_2$
 & $\bigl(\rho_2\langle p_\m,q_{\m,i}\rangle_{i\le N}\,,\,r_\m,\,p_\m,\,\Psi^{(2)}_1\bigr)$
 & $1$
 & $4{\cdot}N{+}1{+}1{+}D$ \\
\rowcolor{phi2bg}
$5$ & $\Psi^{(2)}_2,\Psi^{(3)}_2$
 & $\bigl(\rho_2\langle \hp{K_i}-\mathtt{Mult}(p_\m,q_{\m,i}),\,r_\m\rangle_{i\le N}\,,$ \newline
   $\quad (E_{\m,i})_{i\le N},\,p_\m,\,\Psi^{(2)}_1\bigr)$
 & \cellcolor{paramHL}$\displaystyle\prod_{\m}p_\m$
 & \cellcolor{widthHL}$4{\cdot}N{+}N{+}1{+}D$ \\
\rowcolor{phi2bg}
$6$ & $\Psi^{(3)}_2$
 & $\bigl(\rho_\floor(E_{\m,i}r_\m)_{i\le N}\,,\,   \newline (E_{\m,i})_{i\le N},\, p_\m,\, \Psi^{(2)}_1\bigr)$
 & $1$
 & $N{+}N{+}1{+}D$ \\
\rowcolor{phi2bg}
$7$ & $\Psi^{(3)}_2$
 & $\bigl(\rho_2\langle p_\m,\chi_{\m,i}\rangle_{i\le N}\,,\,(E_{\m,i})_{i\le N},\,\Psi^{(2)}_1\bigr)$
 & $1$
 & $4{\cdot}N{+}N{+}D$ \\
\rowcolor{phi3bg}
$8\dots 7{+}r$ & $\Psi^{(3)}_2,\Psi^{(1)}_3$
 & $\bigl(\rho_2\langle\Psi^{(2)}_1{}^{\balpha^{(i)}}\rangle_{i\le N}\,,\,(c_{\m,i})_{i\le N},\,\Psi^{(2)}_1\bigr)$
 & $1$
 & \cellcolor{widthHL}$4{\cdot}N{+}N{+}D$ \\
\rowcolor{phi3bg}
$8{+}r$ & $\Psi^{(2)}_3$
 & $\rho_2\langle c_{\m,i},\,\Psi^{(2)}_1{}^{\balpha^{(i)}}\rangle_{i\le N}$
 & $1$
 & $4{\cdot}N$ \\
\rowcolor{phi3bg}
$9{+}r$ & $\Psi^{(2)}_3$
 & $\displaystyle\sum_{i\le N}c_{\m,i}\,\Psi^{(2)}_1{}^{\balpha^{(i)}}=Y_\m(\Psi^{(2)}_1)$
 & $1$
 & $1$ \\
\bottomrule
\end{tabular}
\caption{Architecture and parameter magnitudes for the construction of the gridwise polynomial with rational coefficients.}
\label{tab:width-depth-poly}
\end{table}

%% file: chapters/07_Discussion.tex
\section{Conclusion}
\label{sec:discussion-limitations}

This paper has developed a novel fixed-architecture approximation framework based on the Chinese Remainder Theorem, which moves one step further in the domain of super-expressiveness. For Lipschitz targets (Theorem~\ref{thm:main-theorem1}), our construction for the first time enjoys an explicit relation between parameter magnitude and the approximation error. For H\"older-smooth targets (Theorem~\ref{thm:main-theorem2}), our construction also realizes an explicit representation, which shows a smoothness scaling law. The key to our two constructions is that the CRT ensures the accurate estimation of integer magnitude. There is still room for improvement in our constructions. First, the CRT parameter can be enormous. Though it is achievable, there might be redundancy in the construction. Second, like other super-expressive constructions, our result uses nonstandard activations, which narrows the scope of our theory. Our future work will attempt to address these limitations.